\documentclass{article} 
\usepackage{iclr2027_conference,times}

\usepackage{amsmath,amsfonts,bm}

\def\eqref#1{equation~\ref{#1}}

\def\1{\bm{1}}

\DeclareMathAlphabet{\mathsfit}{\encodingdefault}{\sfdefault}{m}{sl}
\SetMathAlphabet{\mathsfit}{bold}{\encodingdefault}{\sfdefault}{bx}{n}

\usepackage{hyperref}
\usepackage[capitalize,noabbrev,nameinlink]{cleveref}
\usepackage{url}

\definecolor{teal}{rgb}{0.040, 0.470, 0.390}
\hypersetup{colorlinks, citecolor=teal, linkcolor=teal}

\title{Scaling Influence Functions in LLMs through Eigenbasis-Corrected One-Bit Gradient \\ Projection}

\author{Jaeseung Heo \\
POSTECH\\
\texttt{jsheo12304@postech.ac.kr} \\
\And
J Rosser \\
University of Oxford \\
\texttt{jrosser@robots.ox.ac.uk} \\
\And
Dongwoo Kim \\
POSTECH \\
\texttt{dongwookim@postech.ac.kr}
}

\usepackage{cleveref}
\usepackage{subcaption}
\usepackage{amsmath,amssymb,amsthm,mathtools}
\usepackage{xcolor,booktabs,tabularx,array}
\usepackage{algorithm,algpseudocode,float}
\usepackage{placeins}
\usepackage{wrapfig}

\newtheorem{proposition}{Proposition}
\crefname{appendix}{appendix}{appendices}
\Crefname{appendix}{Appendix}{Appendices}
\newcommand{\method}{EOGP}
\newcommand{\oinff}{\mathcal{I}}
\newcommand{\inff}{\hat{\mathcal{I}}}
\newcommand{\grad}{\nabla_{\theta}}

\newcommand{\thetastar}{\theta^*}
\newcommand{\Hlambda}{H_\lambda}
\newcommand{\traingrad}[1]{\grad{}\ell_{#1}(\thetastar)}
\newcommand{\querygrad}{\grad{}f}
\newcommand{\hwtrain}{\tilde{t}}
\newcommand{\hwquery}{\tilde{q}}
\newcommand{\hwsecond}{\Sigma}
\newcommand{\compressionerror}{\mathcal{E}}
\newcommand{\expectation}{\mathbb{E}}
\DeclareMathOperator{\tr}{tr}
\DeclareMathOperator{\rank}{rank}
\DeclareMathOperator{\range}{range}

\newcommand{\Hekfac}{\widehat{H}_\lambda}
\newcommand{\ektrain}{\hwtrain^{\mathrm{EK}}}

\newcommand{\methodrandom}{\method{}-R}

\newcommand{\sketch}{\mathsf{S}}

\definecolor{reviewink}{HTML}{825213}
\definecolor{reviewlight}{HTML}{FFF8E8}

\providecommand{\method}{EOGP}
\providecommand{\methodrandom}{\method{}-R}
\providecommand{\thetastar}{\theta^*}
\providecommand{\grad}{\nabla_\theta}
\providecommand{\Hekfac}{\widehat{H}_\lambda}
\providecommand{\hwtrain}{\tilde t}
\providecommand{\hwquery}{\tilde q}
\providecommand{\hwsecond}{\Sigma}
\providecommand{\compressionerror}{\mathcal E}
\providecommand{\expectation}{\mathbb E}
\providecommand{\tr}{\operatorname{tr}}
\providecommand{\rank}{\operatorname{rank}}
\providecommand{\range}{\operatorname{range}}
\providecommand{\sketch}{\mathsf S}

\iclrfinalcopy 
\begin{document}

\maketitle
\lhead{Preprint. Under review.}

\begin{abstract}
Influence functions estimate how individual training examples affect the behavior of large language models (LLMs). Analyzing how training data influence different behaviors of an LLM involves repeated influence computation. Reusing stored training gradients reduces the computational cost, but storing full gradients is prohibitively expensive at LLM scale. We study how to compress these gradients while preserving influence estimates for future queries that are unknown at storage time. Through a worst-case analysis, we characterize the optimal fixed-dimensional linear representation and propose eigenbasis-corrected one-bit gradient projection (\method{}) to approximate it at scale. Specifically, \method{} uses EK-FAC to reduce gradient dimensionality, then applies PCA within the retained subspace to learn compression directions from the training gradients. We then apply one-bit quantization to the resulting coordinates, allowing more coordinates to be retained within a fixed storage budget. On GPT-2, \method{} predicts retraining outcomes more accurately than the evaluated compression baselines while using one-sixteenth of their per-example storage. On OLMo 2 SFT models from 1B to 32B parameters, \method{} remains competitive with the baselines allocated over 100 times as much storage per example. 
\end{abstract}

\section{Introduction}
\label{sec:introduction}

Influence functions are widely used to analyze how training data shape the diverse behaviors of large language models (LLMs). Their applications include studying generalization~\citep{grosse2023studying,ruis2025procedural,kou2026data}, selecting data for fine-tuning~\citep{zhou2024hyperinf,heo2026interaction,chen2026influence}, and identifying training examples that contribute to undesired behavior~\citep{zhang2025correcting,coalson2026if}. By estimating how individual training examples affect a model's predictions~\citep{koh2017understanding}, influence functions provide a common framework for investigating these connections between training data and model behavior.

Each attribution query asks which training examples contribute to a particular model output or behavior. Analyzing different behaviors of a fixed model therefore involves multiple queries, but the training-example gradients used in influence estimation can be reused across them. Computing and storing these gradients once would avoid repeated backward passes over the training data as new queries arise. At LLM scale, however, storing full gradients is costly: a single training-example gradient for an 8B-parameter model requires 16~GB in half precision. Supporting repeated attribution queries efficiently thus calls for compact gradient representations that can be stored and reused.

Existing gradient-compression methods reduce this storage cost through random~\citep{schioppa2024efficient,choe2026your,hu2026grass} or curvature-informed projections~\citep{schioppa2022scaling,choe2026your}, representing each gradient with a small number of coordinates. The central challenge is deciding which information these coordinates should preserve. Since future attribution queries are not known when the representations are constructed, compression must retain information that remains useful across different subsequent influence calculations.

We address this challenge by defining a compression objective based on the discrepancy between influence scores computed from compressed and uncompressed gradients across possible future queries. Under a normalized worst-case influence-error criterion, we show that the top-$k$ PCA coordinates of half-whitened training gradients are optimal among $k$-dimensional linear representations. This result provides a concrete target for constructing reusable gradient representations.

To approximate this target at LLM scale, we propose Eigenbasis-corrected One-bit Gradient Projection (\method{}). Computing PCA directly on full gradients is prohibitively expensive for billion-parameter models. We therefore first reduce their dimensionality using projection directions derived from EK-FAC~\citep{george2018fast}, then apply PCA to the resulting representations to learn the final compression basis. This allows us to refine the representation using the training gradients within a computationally manageable space. We then apply one-bit quantization to the resulting coordinates, allowing more coordinates to be stored within a fixed budget. Together, these steps address which directions to retain and how much precision to allocate to each coordinate.

We evaluate the resulting representations through retraining-based validation and comparison with an uncompressed influence reference. On GPT-2~\citep{radford2019language}, \method{} outperforms the tested compression baselines in both linear datamodeling score (LDS)~\citep{park2023trak} and counterfactual retraining~\citep{bae2024training}. On OLMo~2 SFT models~\citep{olmo20242} from 1B to 32B parameters, \method{} more accurately reproduces the reference influence rankings at matched storage budgets. Even with less than 16~KB per example, it remains competitive with all tested compression baselines allocated over 100 times as much storage. These results show that compact, reusable gradient representations can retain useful attribution information at substantially lower storage cost.

\section{Preliminaries}

\subsection{Influence Functions}
\label{sec:influence}

Influence functions~\citep{koh2017understanding} estimate how a model's behavior would change if a training example were removed and the model retrained. 
Let $\theta\in\mathbb R^d$ denote the model parameters, $\thetastar\in\mathbb R^d$ their trained values, and $\ell_i:\mathbb R^d\to\mathbb R$ the loss function for training example $i$.
The corresponding \emph{training gradient} is $\traingrad{i}\in\mathbb R^d$, and $H\in\mathbb R^{d\times d}$ denotes the Hessian of the average training loss at $\thetastar$.

For a scalar target measurement $f:\mathbb R^d\to\mathbb R$, we define the \emph{query gradient} $\querygrad\in\mathbb R^d$ as its gradient with respect to $\theta$, evaluated at $\thetastar$. For example, a query about the loss on a target document $z$ uses $\querygrad=\grad{}\ell_{z}(\thetastar)$. The influence score of training example $i$ for this query is
\begin{equation}
\oinff{}(i) \coloneqq \querygrad^{\top}H^{-1}\traingrad{i}
\in\mathbb R.
\label{eq:if}
\end{equation}

For neural networks, the loss Hessian $H$ can be singular, so $H^{-1}$ may not exist. Prior work addresses this by using the generalized Gauss--Newton (GGN) matrix $G\in\mathbb R^{d\times d}$, with $G\succeq0$, together with damping~\citep{martens2020new,bae2022if,grosse2023studying}. For $\lambda>0$, we define $\Hlambda\coloneqq G+\lambda I_d\in\mathbb R^{d\times d}$, where $I_d$ is the $d\times d$ identity matrix. Since $\Hlambda\succ0$, the practical influence score is
\begin{equation}
\inff{}(i) \coloneqq \querygrad^{\top}\Hlambda^{-1}\traingrad{i}
\in\mathbb R.
\label{eq:practical_if}
\end{equation}

For a fixed model, the training gradient $\traingrad{i}$ does not depend on the query and can therefore be reused across queries. In this work, we aim to store these gradients so that subsequent influence calculations do not require repeated backward passes over the training data.

\subsection{Kronecker-Factored Curvature Approximation} \label{sec:ekfac}

\paragraph{GGN--Fisher equivalence.} Storing and factorizing the full GGN is impractical at LLM scale, so prior work has used structured approximations of the Fisher information matrix for scalable influence estimation~\citep{grosse2023studying}. For softmax cross-entropy on logits, the loss commonly used to train autoregressive LLMs, the GGN coincides with the Fisher $F$~\citep{martens2020new}. Writing $g(x,y)\coloneqq\nabla_\theta\ell(x,y;\thetastar)$, this matrix is 
\begin{equation} 
G=F=\expectation_{x\sim\mathcal{D}_X}\expectation_{\tilde y\sim p_{\thetastar}(\cdot\mid x)}\!\left[g(x,\tilde y)g(x,\tilde y)^{\top}\right], 
\label{eq:fisher_ggn} 
\end{equation} 
where $\mathcal{D}_X$ is the empirical distribution of training inputs. Thus, the curvature is a gradient second moment with labels drawn from the model's predictive distribution.

\paragraph{K-FAC approximation.} K-FAC~\citep{martens2015optimizing} exploits the layer structure of neural networks to approximate this Fisher matrix efficiently. In its block-diagonal form, it neglects cross-layer blocks and approximates each layer's block by a Kronecker product. For a linear layer, the two Kronecker factors are the second-moment matrices of its input activations $a$ and the loss gradients $\delta$ with respect to its outputs.
\begin{equation}
G_{\mathrm{layer}} = \expectation\!\left[
(aa^{\top})\otimes(\delta\delta^{\top})
\right] \approx A\otimes S, \qquad A\coloneqq\expectation[aa^{\top}], \quad S\coloneqq\expectation[\delta\delta^{\top}].
\label{eq:kfac_factorization}
\end{equation}
Let $Q_A$ and $Q_S$ be the orthonormal eigenbases of $A$ and $S$, respectively. Then $Q_A\otimes Q_S$ is an eigenbasis of the K-FAC approximation $A\otimes S$. This structure allows the basis to be obtained from two smaller eigendecompositions and applied through matrix multiplications without materializing the full basis matrix.

\paragraph{EK-FAC approximation.} EK-FAC~\citep{george2018fast,grosse2023studying} corrects the eigenvalues of the K-FAC approximation while retaining its block structure and layerwise eigenbases. It obtains the corrected values by expressing gradients in these bases and taking the mean square of each coordinate. With labels sampled as in \Cref{eq:fisher_ggn}, the resulting approximation is 
\begin{equation} 
G\approx Q\Lambda Q^{\top}, 
\label{eq:ekfac_eigendecompose} 
\end{equation} 
where $Q$ is block diagonal, with each block given by the corresponding layer's K-FAC eigenbasis $Q_A\otimes Q_S$, and $\Lambda$ is the diagonal matrix of corrected eigenvalues.

\section{Method}
\label{sec:method}

This section presents Eigenbasis-corrected One-bit Gradient Projection (\method{}), which stores compact representations of training gradients for reuse in influence estimation. We first formulate gradient compression as minimizing a normalized worst-case influence error and show that storing the top-$k$ PCA coordinates of half-whitened training gradients is optimal among $k$-dimensional linear representations (\S\ref{sec:objective}). We then approximate this target using a projection onto the EK-FAC eigenbasis followed by PCA within the retained subspace (\S\ref{sec:two_stage}). Finally, we store the resulting training coordinates as one sign bit each, together with a scale per module (\S\ref{sec:onebit}). 
\Cref{fig:pipeline} illustrates the pipeline, and \Cref{app:algorithm} gives the full procedure.

\begin{figure}[t]
\centering
\includegraphics[width=\textwidth]{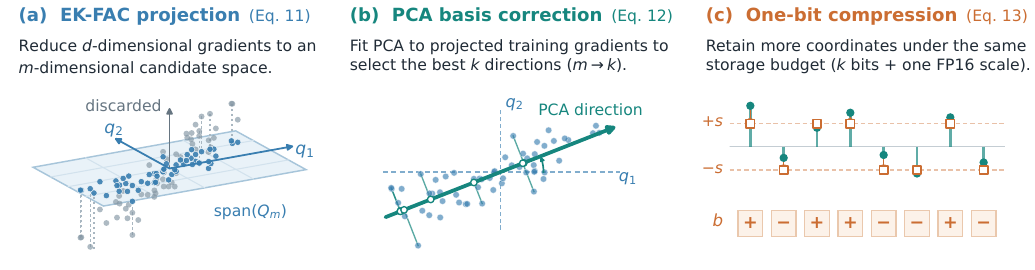}
\caption{Overview of \method{}. (a) Training gradients are projected onto an EK-FAC subspace and half-whitened. (b) PCA corrects the compression basis within this subspace. (c) The resulting coordinates are stored as one sign bit per coordinate together with a shared scale. Query gradients undergo the same transformations without one-bit quantization, and influence scores are computed as scaled inner products with the stored bits.}
\label{fig:pipeline}
\end{figure}

\subsection{What to Store: PCA of Half-Whitened Gradients}
\label{sec:objective}

\paragraph{Compression objective.} We first ask which $k$-dimensional linear representation best preserves the influence scores in \Cref{eq:practical_if} when queries are unknown at storage time. We consider linear compression methods that store $V^{\top}\traingrad{i}\in\mathbb{R}^{k}$ per training example, where $V\in\mathbb{R}^{d\times k}$ defines the compression. A reconstruction matrix $M\in\mathbb{R}^{d\times k}$ maps these coordinates back to parameter space, giving the influence estimate $\querygrad^{\top}MV^{\top}\traingrad{i}$.

We evaluate compression by the discrepancy between the uncompressed influence score and the estimate obtained from the stored coordinates. Since this discrepancy scales linearly with the query gradient, we constrain query magnitude to obtain a finite worst-case error.

The influence score in \Cref{eq:practical_if} weights the interaction between the query and training gradients by $\Hlambda^{-1}$, giving greater weight to components along directions of lower curvature. The standard Euclidean norm treats all directions equally and therefore does not capture this directional weighting. To account for how query components enter the influence computation, we use the $\Hlambda^{-1}$-weighted Euclidean norm and define its unit ball $\mathcal B$ as
\begin{align}
\|q\|_{\Hlambda^{-1}}
\coloneqq \sqrt{q^{\top}\Hlambda^{-1}q}, \quad\quad
\mathcal B
\coloneqq
\left\{q\in\mathbb R^d:
\|q\|_{\Hlambda^{-1}}\le 1
\right\}.    
\label{eq:ball}
\end{align}
Our objective is the worst-case squared influence error over query gradients in $\mathcal B$, averaged over training examples:
\begin{equation}
\compressionerror(V,M)
\coloneqq
\expectation_i\Big[
\sup_{\querygrad\in\mathcal B}
\big|
\querygrad^{\top}
\big(\Hlambda^{-1}-MV^{\top}\big)
\traingrad{i}
\big|^2
\Big].
\label{eq:criterion}
\end{equation}
This criterion allows us to optimize the stored representation without knowing which attribution queries will be evaluated later.

\paragraph{Optimal linear representation.} To characterize the optimal representation, we define the half-whitened training and query gradients $\hwtrain_i \coloneqq \Hlambda^{-1/2}\traingrad{i}$ and $\hwquery \coloneqq \Hlambda^{-1/2}\querygrad$. Influence is then the inner product $\langle \hwquery,\hwtrain_i\rangle$.

\begin{proposition}[Optimality of half-whitened PCA]
\label{prop:optimality}
Let $1 \le k \le d$, and let the columns of $U \in \mathbb{R}^{d \times k}$ be orthonormal eigenvectors corresponding to the $k$ largest eigenvalues of the uncentered second moment $\hwsecond \coloneqq \expectation_i[\hwtrain_i\hwtrain_i^{\top}]$. Then \Cref{eq:criterion} is minimized by $V=M=\Hlambda^{-1/2}U$, i.e., by storing $U^{\top}\hwtrain_i$ and estimating influence as its inner product with $U^{\top}\hwquery$.
\end{proposition}

Thus, storing the top-$k$ PCA coordinates of half-whitened training gradients minimizes influence-estimation error under \Cref{eq:criterion}. The proof is provided in \Cref{app:optimality}.

\subsection{Two-Stage Projection: PCA within an EK-FAC Subspace}

Performing PCA directly on half-whitened training gradients is impractical at LLM scale,
so we approximate it in two stages. The first stage uses the gradient statistics provided by EK-FAC to construct an $m$-dimensional candidate subspace. The second stage applies PCA to the training gradients projected into this subspace to select the final $k$ directions, where $m>k$.

We apply both stages independently to each linear layer (module). Throughout this subsection, we suppress module indices and use $d$, $m$, and $k$ for the corresponding per-module dimensions.

\paragraph{First stage: subspace selection with EK-FAC.}
Let $C\coloneqq\expectation_i[\traingrad{i}\traingrad{i}^{\top}]\in\mathbb{R}^{d\times d}$ denote the uncentered second moment of the training gradients, also known as the empirical Fisher~\citep{martens2020new}. Applying the PCA criterion of \Cref{prop:optimality} within each module requires the leading eigenvectors of
\begin{equation}
\hwsecond
=\Hlambda^{-1/2}C\Hlambda^{-1/2}.
\label{eq:sigma_exact}
\end{equation}

To construct a tractable surrogate for $\hwsecond$, we approximate both the half-whitening operator and the training-gradient second moment. We begin with the half-whitening operator, which depends on the curvature matrix $G$. For softmax cross-entropy, $G$ coincides with the true Fisher and can therefore be approximated using EK-FAC, as reviewed in \Cref{sec:ekfac}. Writing $G\approx Q\Lambda Q^{\top}$, where $\Lambda=\mathrm{diag}(\lambda_1,\ldots,\lambda_d)$ contains the corrected eigenvalues, gives
\begin{equation}
\Hlambda^{-1/2}
=(G+\lambda I)^{-1/2}
\approx Q(\Lambda+\lambda I)^{-1/2}Q^{\top}
\eqqcolon\Hekfac^{-1/2}.
\label{eq:ekfac_half_whitening}
\end{equation}

For candidate subspace selection, we additionally use $Q\Lambda Q^{\top}$ as a surrogate for $C$. 
This is a separate approximation: although the Fisher and the empirical Fisher $C$ are both gradient second moments, the Fisher averages over the model's predictive label distribution for each input, whereas $C$ uses the corresponding training label.
Combining this surrogate with the approximate half-whitening in \Cref{eq:ekfac_half_whitening} yields
\begin{equation}
\begin{aligned}
\Sigma_{\mathrm{proxy}}
&\coloneqq
\Hekfac^{-1/2}Q\Lambda Q^{\top}\Hekfac^{-1/2} \\
&=
Q\,\mathrm{diag}\!\left(
\frac{\lambda_1}{\lambda_1+\lambda},
\ldots,
\frac{\lambda_d}{\lambda_d+\lambda}
\right)Q^{\top}.
\end{aligned}
\label{eq:sigma_ekfac}
\end{equation}
Since $\lambda_j/(\lambda_j+\lambda)$ increases with $\lambda_j$ for $\lambda>0$, the $m$ EK-FAC eigenvectors with the largest corrected eigenvalues form a top-$m$ PCA basis for $\Sigma_{\mathrm{proxy}}$. We can therefore construct this candidate subspace directly from the existing EK-FAC statistics, without separately estimating $C$ or performing a full-space eigendecomposition.

Let $Q_m\in\mathbb R^{d\times m}$ contain these eigenvectors and $\Lambda_m\in\mathbb R^{m\times m}$ be the diagonal matrix of their corrected eigenvalues.
Projecting the EK-FAC half-whitened training gradients onto this basis gives the first-stage coordinates:
\begin{equation}
\begin{aligned}
\hwtrain_i^{(1)}
&\coloneqq
Q_m^{\top}\Hekfac^{-1/2}\traingrad{i} \\
&=
(\Lambda_m+\lambda I)^{-1/2}
Q_m^{\top}\traingrad{i}
\in\mathbb R^m.
\end{aligned}
\label{eq:first_stage}
\end{equation}
The Kronecker factorization $Q=Q_A\otimes Q_S$ allows these coordinates to be computed through matrix multiplications with the two smaller factors, without forming a dense $d\times m$ projection matrix.

\paragraph{Second stage: subspace refinement with PCA.}
\label{sec:two_stage}
The EK-FAC basis is inherited from the K-FAC approximation $A\otimes S$, which uses separate second moments of activations and backpropagated gradients. Although EK-FAC corrects the diagonal second moments in this basis, it keeps the basis fixed and leaves off-diagonal terms unmodeled. Consequently, the retained axes need not align with the principal directions of the half-whitened training gradients.

To account for the second-moment structure among the retained coordinates, the first stage keeps $m>k$ EK-FAC axes. We then find the optimal $k$-dimensional subspace within their span by applying PCA to the full second moment of the first-stage coordinates,
$
\hwsecond^{(1)}
\coloneqq
\expectation_i\!\left[
\hwtrain_i^{(1)}(\hwtrain_i^{(1)})^\top
\right],
$
including its off-diagonal entries.

Let $P\in\mathbb R^{m\times k}$ contain the orthonormal eigenvectors of $\hwsecond^{(1)}$ associated with its $k$ largest eigenvalues. Combining this PCA projection with the first-stage transformation gives the final coordinates:
\begin{equation}
\begin{aligned}
\hwtrain_i^{(2)}
&\coloneqq P^\top\hwtrain_i^{(1)} \\
&=(Q_mP)^\top\Hekfac^{-1/2}\traingrad{i}
\in\mathbb R^k.
\end{aligned}
\label{eq:second_stage}
\end{equation}

The following proposition shows that this correction yields a reconstruction error less than or equal to that obtained using the leading $k$ EK-FAC axes.

\begin{proposition}[Eigenbasis correction]
\label{prop:rayleigh_ritz}
Let $\ektrain_i\coloneqq\Hekfac^{-1/2}\traingrad{i}$ denote the EK-FAC half-whitened training gradients, $U_\star\coloneqq Q_mP$ the corrected basis, and $Q_k$ the basis formed by the leading $k$ EK-FAC axes retained in $Q_m$. Projecting these gradients onto $\operatorname{span}(U_\star)$ has the following guarantees:

\begingroup
\setlength{\leftmargini}{2em}
\setlength{\topsep}{2pt}
\setlength{\partopsep}{0pt}
\setlength{\itemsep}{2pt}
\setlength{\parsep}{0pt}
\setlength{\parskip}{0pt}
\renewcommand{\labelenumi}{\textup{(\roman{enumi})}}
\begin{enumerate}

\item The mean squared reconstruction error is less than or equal to that obtained by projecting onto the leading $k$ EK-FAC axes. Formally,
\[
\expectation_i\!\left[\|\ektrain_i-U_\star U_\star^{\top}\ektrain_i\|_2^2\right]
\le
\expectation_i\!\left[\|\ektrain_i-Q_kQ_k^{\top}\ektrain_i\|_2^2\right].
\]
In fact, it is minimal among all $k$-dimensional subspaces of $\operatorname{span}(Q_m)$.

\item If $\operatorname{span}(Q_m)$ contains a top-$k$ principal subspace of $\ektrain_i$, the error equals the minimum attained by full-space $k$-dimensional PCA within that module.

\end{enumerate}
\endgroup
\end{proposition}

Thus, retaining $m>k$ candidate axes allows the final subspace to be optimized within a larger span while storing only $k$ coordinates per training example. The proof is provided in \Cref{app:eigenbasis_correction}.

\paragraph{A scalable alternative to PCA.}
PCA minimizes reconstruction error within the retained span for a fixed output dimension $k$. However, storing and applying its dense correction matrix $P\in\mathbb{R}^{m\times k}$ becomes costly as $k$ grows. To exploit larger per-example storage budgets without this overhead, we introduce \methodrandom{}, which replaces the PCA correction with an implicitly applied subsampled randomized Hadamard transform (SRHT)~\citep{tropp2011improved}, $\sketch\in\mathbb{R}^{k\times m}$:
\[
\hwtrain_i^{(2,\mathrm{R})}
= \sketch\hwtrain_i^{(1)}.
\]
The construction and normalization of $\sketch$ are detailed in \Cref{app:srht}.

When the output dimension $k$ is sufficiently large, SRHT preserves norms within a fixed low-dimensional subspace up to bounded distortion with high probability~\citep{tropp2011improved}. For query and training representations in this subspace, this guarantee also controls inner-product distortion and the resulting influence-score error introduced by sketching. This variant offers a trade-off between PCA's fixed-$k$ reconstruction optimality and scalability to larger output dimensions without a dense correction matrix.

\subsection{One-Bit Quantization} \label{sec:onebit}

Even after projection, storing every coordinate in floating point limits the number of directions we can retain under a fixed per-example storage budget. We apply one-bit quantization to both \method{} and \methodrandom{} to accommodate more projected coordinates within this budget.

For each module $u$, let $\hwtrain_{i,u}^{(2)}\in\mathbb{R}^{k_u}$ denote the projected training-gradient coordinates. We store a binary sign vector and one scale: \begin{equation} b_{i,u}\coloneqq\operatorname{sign}\big(\hwtrain_{i,u}^{(2)}\big)\in\{\pm1\}^{k_u}, \qquad s_{i,u}\coloneqq\frac{1}{k_u}\big\|\hwtrain_{i,u}^{(2)}\big\|_1. \label{eq:onebit} \end{equation} The scale captures the average magnitude of the coordinates and is chosen to minimize $\|\hwtrain_{i,u}^{(2)}-s\,b_{i,u}\|_2^2$ over $s$~\citep{rastegari2016xnor}. Each module requires $k_u$ sign bits and one half-precision scale per training example. The complete representation contains $\sum_u k_u$ coordinates.

At query time, we apply the same projection to the query gradient and keep its coordinates $\hwquery_u^{(2)}$ unquantized. We estimate influence by \begin{equation} \widehat{\mathcal{I}}(i)\coloneqq\sum_u s_{i,u}\,\big\langle\hwquery_u^{(2)},b_{i,u}\big\rangle. \label{eq:onebit_influence} \end{equation} This computes influence from the stored representations without recomputing training gradients.

\section{Experiments}
\label{sec:experiments}

We evaluate whether compact gradient representations preserve useful training-data attributions under limited storage budgets. We first compare against retraining outcomes on GPT-2 (\S\ref{sec:exp_retraining}), then measure fidelity to EK-FAC influence on billion-parameter LLMs (\S\ref{sec:exp_llms}). We examine the contributions of the projection and quantization choices (\S\ref{sec:exp_ablation}), and report the storage and computation required to build the representations (\S\ref{sec:exp_cost}).

\subsection{Validation Against Retraining}
\label{sec:exp_retraining}

\begin{figure}[t]
\centering
\includegraphics[width=\textwidth]{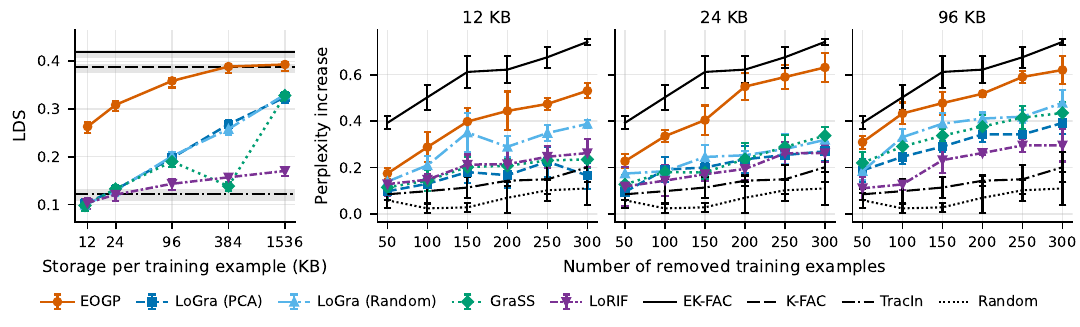}
\caption{Validation against retraining-based ground truth on GPT-2 across storage budgets: linear datamodeling score (left) and counterfactual retraining (right). Black lines denote uncompressed influence functions with different Hessian approximations and the random-removal baseline.}
\label{fig:lds_counterfactual}
\end{figure}

\paragraph{Setup and metrics.} We evaluate GPT-2~\citep{radford2019language} on WikiText-2~\citep{merity2016pointer} at matched per-example storage budgets of $\{12,24,96,384,1536\}$ KB. An uncompressed half-precision gradient requires 162 MB per example. Compression baselines include LoGra~\citep{choe2026your} with PCA and random initialization, GraSS~\citep{hu2026grass}, and LoRIF~\citep{li2026lorif}. We use EK-FAC influence, K-FAC influence, and TracIn~\citep{pruthi2020estimating} as uncompressed references, and random removal as an additional baseline for the counterfactual experiment.

We use two retraining-based metrics. The linear datamodeling score (LDS)~\citep{park2023trak} evaluates how well influence estimates predict retraining outcomes. It measures the rank correlation between the estimated influences of randomly sampled training subsets and the outcomes measured after retraining the model on each subset. Counterfactual retraining~\citep{bae2024training} removes the examples ranked most influential by each method and measures the resulting increase in validation perplexity after retraining. Higher LDS indicates closer agreement with retraining outcomes, while a larger perplexity increase indicates a greater impact of removing the selected examples. Further metric and protocol details are provided in \Cref{app:exp_retraining}.

\paragraph{Results.} \Cref{fig:lds_counterfactual} shows that \method{} improves retraining-based attribution quality over the compression baselines, with the largest gains at small storage budgets. In LDS, \method{} at 96 KB exceeds all compression-based baselines at 1,536 KB, despite using 16 times less per-example storage. At 384 KB, it achieves LDS close to the uncompressed K-FAC reference while using over 400 times less per-example storage than the 162 MB uncompressed representation. 
The counterfactual experiment likewise shows that the examples selected by \method{} produce larger perplexity increases than those selected by the compression baselines.

\subsection{Fidelity to EK-FAC at LLM Scale}
\label{sec:exp_llms}

\paragraph{Setup.} We evaluate the SFT checkpoints of OLMo 2~\citep{olmo20242} with 1B, 7B, 13B, and 32B parameters. Training examples come from the T\"ulu 3 SFT mixture~\citep{lambert2024tulu}, and queries come from a held-out split. Uncompressed half-precision gradients occupy approximately 2--62 GB per example across these models. 
We compare compression methods at budgets ranging from a few kilobytes to megabytes per example, corresponding to compression ratios on the order of $10^4$--$10^6$ relative to uncompressed half-precision gradients.
Further evaluation details are given in \Cref{app:exp_llms}.
\begin{wrapfigure}{R}{0.5\textwidth}
\centering
\includegraphics[width=\linewidth]{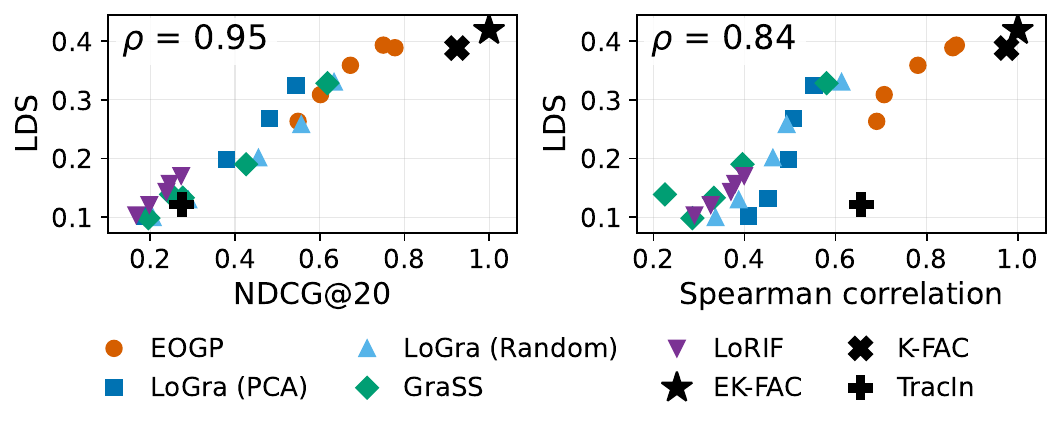}
\caption{LDS versus EK-FAC fidelity on GPT-2, measured by NDCG@20 (left) and Spearman correlation (right). $\rho$ is the Spearman correlation across configurations from \Cref{fig:lds_counterfactual}.}
\label{fig:lds_vs_EKFAC}
\vspace{-10px}
\end{wrapfigure}

\paragraph{Metrics and reference.}
Since retraining-based evaluation is computationally prohibitive at this scale, we use agreement with EK-FAC influence rankings as a proxy for attribution quality. We choose EK-FAC as the reference because it performs best in our GPT-2 retraining-based evaluation among the evaluated methods that scale to billion-parameter LLMs (\Cref{sec:exp_retraining}), consistent with prior comparisons~\citep{choe2026your}. For each query, we measure overall ranking agreement using Spearman correlation and agreement near the top of the ranking using NDCG@20~\citep{jarvelin2002cumulated}, then average both metrics over queries.

To examine whether these agreement metrics reflect retraining-based attribution quality, \Cref{fig:lds_vs_EKFAC} shows their relationship with LDS on GPT-2.
Each point corresponds to an evaluated configuration from \Cref{fig:lds_counterfactual}. The x-axis shows agreement between the influence rankings produced by each configuration and the EK-FAC reference rankings, measured by NDCG@20 (left) or Spearman correlation (right). The y-axis shows LDS, which measures how well each configuration's influence estimates predict retraining outcomes.
We report the Spearman correlations ($\rho$) between the x-axis metric and LDS across configurations: 0.95 for NDCG@20 and 0.84 for overall ranking agreement. 
These positive correlations indicate that configurations whose rankings more closely match EK-FAC generally also better predict the retraining outcomes, providing empirical support for using EK-FAC ranking agreement as a proxy for attribution quality.

\begin{figure}[t]
\centering
\includegraphics[width=.99\textwidth]{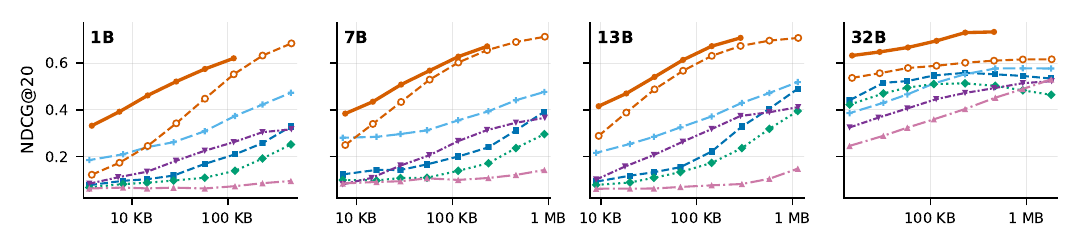}
\includegraphics[width=.99\textwidth]{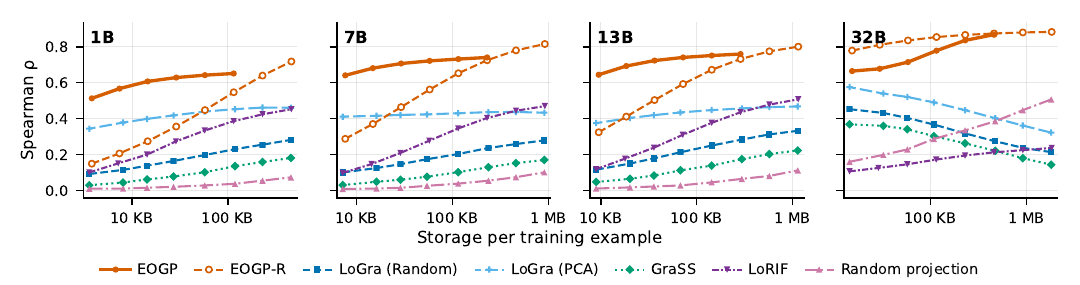}

\caption{NDCG@20 (top) and Spearman correlation (bottom) of gradient compression methods against EK-FAC influence functions on OLMo 2 SFT models from 1B to 32B parameters.}
\label{fig:ekfac}
\end{figure}

\paragraph{Results.} \Cref{fig:ekfac} shows substantial gains over the compression baselines across the evaluated model sizes, particularly at small storage budgets. Fidelity generally improves as more storage is allocated, and continues to increase at the largest tested budgets on the 32B model. The improvements appear in both the overall ranking and the top-ranked examples. These results are obtained at compression ratios of approximately $10^4$--$10^6$, supporting the use of compact representations even when storing uncompressed gradients would require tens of gigabytes per example. 
In \Cref{app:shared_curvature}, we additionally evaluate LoGra and GraSS with the same EK-FAC half-whitening and show that \method{} consistently outperforms both variants.

\subsection{Ablation Studies}
\label{sec:exp_ablation}
\begin{wrapfigure}{R}{0.45\textwidth}
\centering
\includegraphics[width=\linewidth]{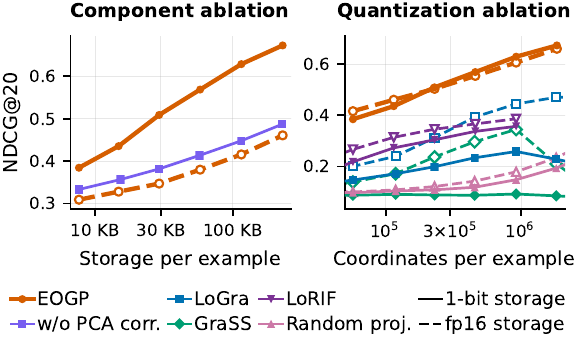}
\caption{Component ablation of \method{} (left) and comparison of 16-bit and 1-bit storage at matched coordinate counts (right).}
\label{fig:ablation}
\vspace{-30px}
\end{wrapfigure}%
We examine the effects of PCA basis correction (\S\ref{sec:two_stage}) and one-bit quantization (\S\ref{sec:onebit}) on OLMo~2 7B SFT. We compare methods at matched per-example storage budgets (\Cref{fig:ablation}, left) and matched numbers of stored coordinates (\Cref{fig:ablation}, right). These comparisons assess performance under a fixed storage constraint and representation quality at a fixed dimension, respectively.

\paragraph{Comparison at matched storage budgets.}
At matched storage budgets, \method{} consistently achieves higher NDCG@20 than direct selection of EK-FAC basis directions (\Cref{fig:ablation}, left). Both variants use one-bit quantization, demonstrating the benefit of refining the compression basis through PCA of the first-stage coordinates. For \method{}, one-bit representations also outperform FP16 representations by allowing more coordinates to be retained within the same storage budget.
We provide further evidence of \method{}'s robustness to one-bit quantization and discuss a possible explanation in \Cref{app:onebit_analysis}.

\paragraph{Comparison at matched coordinate counts.}
At matched coordinate counts, \method{} achieves higher NDCG@20 than every baseline stored at the same precision, under both FP16 and one-bit storage (\Cref{fig:ablation}, right).
This advantage at fixed dimension supports the quality of the two-stage projection.
Within \method{}, one-bit and FP16 representations achieve similar NDCG@20, indicating that one-bit quantization substantially reduces storage with little loss in ranking agreement.

\subsection{Storage and Computational Costs}
\label{sec:exp_cost}

\paragraph{Per-example storage.} 
With $k_u$ coordinates in module $u$, storing the compressed representation of one training example across all modules requires $\sum_u(\lceil k_u/8\rceil+2)$ bytes, including packed sign bits and a two-byte scale per module. At $k_u=2{,}048$, the 32B model requires approximately 115.6~KB per training example, compared with 62.4~GB for an uncompressed half-precision gradient.

\paragraph{Shared storage overhead.} The Kronecker eigenbases, half-whitening weights, and PCA correction matrices are shared across all training examples and queries and are accounted for separately from the per-example payloads. For the 32B model, each module's PCA correction matrix requires approximately 1~GB when its first-stage and final output dimensions are $m_u=262{,}144$ and $k_u=2{,}048$, respectively. These shared storage costs do not grow with the number of stored examples. \Cref{app:storage_accounting} provides formulas for per-example and shared storage.

\paragraph{Store-construction time.} On a single NVIDIA B200, our implementation of \method{} achieves store-construction times comparable to LoGra. For the 32B model at approximately 2,048 output coordinates per module, the per-example time is 0.247 seconds for \method{} and 0.278 seconds for LoGra. These costs include gradient computation, projection, and disk I/O, including correction-matrix loading and transfer for \method{}. Fitting and saving the PCA correction matrices from cached first-stage coordinates takes an additional 18.4 minutes as a one-time preprocessing cost. Implementation optimizations and detailed timing protocols are provided in \Cref{app:timing}.
\section{Related Work}
\label{sec:related_work}

\paragraph{Scalable influence estimation.} Influence estimation in large neural networks relies on efficient approximations to inverse-curvature products~\citep{koh2017understanding,guo2021fastif,kwon2024datainf}. Arnoldi iteration approximates a dominant Hessian eigenspace to compute influence in a reduced space~\citep{schioppa2022scaling}. K-FAC and EK-FAC use structured curvature approximations~\citep{martens2015optimizing,george2018fast}, with EK-FAC enabling influence analysis at LLM scale~\citep{grosse2023studying}.

\paragraph{Gradient compression for data attribution.} TRAK scales attribution through random projections~\citep{park2023trak}, while LESS and its quantized extension QLESS select instruction data using projected gradient similarities rather than the inverse-curvature-based influence scores studied here~\citep{xia2024less,ananta2025qless}. For influence estimation, TrackStar~\citep{chang2025scalable} uses curvature-corrected gradients for influence retrieval at LLM scale. LoGra uses Kronecker-factored projections with random or PCA-based initialization~\citep{choe2026your}. GraSS combines gradient sparsification with sparse random projection~\citep{hu2026grass}. LoRIF stores projected gradients as low-rank factors and approximates inverse curvature via truncated SVD~\citep{li2026lorif}.
\section{Conclusion}
\label{sec:conclusion}
We present \method{} for building compact gradient stores for influence estimation across future queries. Under a normalized worst-case influence-error criterion, we show that PCA of half-whitened training gradients is optimal among fixed-dimensional linear representations. \method{} approximates this target through PCA correction within a retained EK-FAC subspace and uses scaled one-bit quantization to retain more coordinates within a fixed storage budget. Our results show that \method{} preserves useful attribution information with substantially less storage than existing compression methods, making reusable gradient stores more practical at LLM scale.

\bibliography{ref}
\bibliographystyle{iclr2027_conference}

\appendix
\section{Theoretical Details}
\label[appendix]{app:theory}

\subsection{Optimality of Half-Whitened PCA}
\label[appendix]{app:optimality}

We prove \Cref{prop:optimality} for a fixed symmetric positive-definite matrix $\Hlambda$ and half-whitened gradients $\hwtrain_i=\Hlambda^{-1/2}\traingrad{i}$ satisfying $\expectation_i\|\hwtrain_i\|_2^2<\infty$. The supremum is over the full query-gradient ball in \Cref{eq:ball}.

\begin{proof}
Let $\hwsecond=\expectation_i[\hwtrain_i\hwtrain_i^{\top}]$, with eigenvalues $\mu_1\ge\cdots\ge\mu_d\ge0$, and let $1\le k\le d$. We show that the minimum of \Cref{eq:criterion} is $\sum_{j=k+1}^d\mu_j$ and that the proposed encoder--decoder pair attains it.

\paragraph{Reduction to low-rank reconstruction.}
For arbitrary $V,M\in\mathbb{R}^{d\times k}$, define
\begin{equation}
R\coloneqq\Hlambda^{1/2}MV^{\top}\Hlambda^{1/2},
\qquad \rank(R)\le k.
\label{eq:app_reconstruction_map}
\end{equation}
Write $\hwquery=\Hlambda^{-1/2}\querygrad$. Since $\traingrad{i}=\Hlambda^{1/2}\hwtrain_i$ and $\querygrad=\Hlambda^{1/2}\hwquery$, the influence error satisfies
\begin{equation}
\querygrad^{\top}\big(\Hlambda^{-1}-MV^{\top}\big)\traingrad{i}
=\hwquery^{\top}(I-R)\hwtrain_i.
\end{equation}
The change of variables between $\querygrad$ and $\hwquery$ is bijective and maps the normalized query-gradient ball to the Euclidean unit ball. Therefore,
\begin{equation}
\compressionerror(V,M)
=\expectation_i\Big[\sup_{\|\hwquery\|_2\le1}\big|\hwquery^{\top}(I-R)\hwtrain_i\big|^2\Big]
=\expectation_i\big\|(I-R)\hwtrain_i\big\|_2^2.
\label{eq:app_reconstruction_objective}
\end{equation}
The last equality follows from Cauchy--Schwarz, with equality for a unit query in the direction of $(I-R)\hwtrain_i$ whenever this residual is nonzero. If it is zero, both sides are zero. Conversely, every matrix $R$ of rank at most $k$ admits a factorization $R=XY^{\top}$ with $X,Y\in\mathbb{R}^{d\times k}$, padding with zero columns when needed. Choosing $M=\Hlambda^{-1/2}X$ and $V=\Hlambda^{-1/2}Y$ realizes this $R$ in \Cref{eq:app_reconstruction_map}. Thus optimizing over $V,M$ is equivalent to optimizing \Cref{eq:app_reconstruction_objective} over all rank-at-most-$k$ matrices $R$.

\paragraph{Orthogonal projection suffices.}
Let $\mathcal{S}=\range(R)$, and let $\Pi_{\mathcal{S}}$ be the orthogonal projector onto $\mathcal{S}$. For each example, $(I-\Pi_{\mathcal{S}})\hwtrain_i$ is orthogonal to $\mathcal{S}$, whereas $(\Pi_{\mathcal{S}}-R)\hwtrain_i$ belongs to $\mathcal{S}$. Consequently,
\begin{equation}
\begin{aligned}
\big\|(I-R)\hwtrain_i\big\|_2^2
&=\big\|(I-\Pi_{\mathcal{S}})\hwtrain_i\big\|_2^2+\big\|(\Pi_{\mathcal{S}}-R)\hwtrain_i\big\|_2^2\\
&\ge\big\|(I-\Pi_{\mathcal{S}})\hwtrain_i\big\|_2^2.
\end{aligned}
\label{eq:app_projection_dominance}
\end{equation}
Taking expectations and using $\Pi_{\mathcal{S}}^{\top}=\Pi_{\mathcal{S}}=\Pi_{\mathcal{S}}^2$ gives
\begin{equation}
\compressionerror(V,M)
\ge\tr\big((I-\Pi_{\mathcal{S}})\hwsecond\big)
=\tr(\hwsecond)-\tr(\Pi_{\mathcal{S}}\hwsecond).
\label{eq:app_energy_bound}
\end{equation}

\paragraph{The principal subspace attains the minimum.}
This PCA optimality step is a consequence of the classical Eckart--Young theorem~\citep{eckart1936approximation}.
Write $\hwsecond=\sum_{j=1}^d\mu_j u_j u_j^{\top}$ in an orthonormal eigenbasis. Set $\alpha_j=u_j^{\top}\Pi_{\mathcal{S}}u_j=\|\Pi_{\mathcal{S}}u_j\|_2^2$. Because $\Pi_{\mathcal{S}}$ is an orthogonal projector of rank at most $k$, these weights satisfy $0\le\alpha_j\le1$ and $\sum_{j=1}^d\alpha_j=\tr(\Pi_{\mathcal{S}})\le k$. It follows that
\begin{equation}
\tr(\Pi_{\mathcal{S}}\hwsecond)
=\sum_{j=1}^d\mu_j\alpha_j
\le\sum_{j=1}^k\mu_j\alpha_j+\mu_k\Big(k-\sum_{j=1}^k\alpha_j\Big)
\le\sum_{j=1}^k\mu_j.
\label{eq:app_topk_energy}
\end{equation}
The first inequality uses $\mu_j\le\mu_k$ for $j>k$ and the bound on the sum of the weights. The second uses $\mu_j\ge\mu_k$ for $j\le k$ and $\alpha_j\le1$. Combining \Cref{eq:app_energy_bound,eq:app_topk_energy}, every encoder--decoder pair satisfies
\begin{equation}
\compressionerror(V,M)\ge\tr(\hwsecond)-\sum_{j=1}^k\mu_j=\sum_{j=k+1}^d\mu_j.
\end{equation}

Now take $U=[u_1,\ldots,u_k]$ and $V=M=\Hlambda^{-1/2}U$. Then $R=UU^{\top}$, so
\begin{equation}
\compressionerror\big(\Hlambda^{-1/2}U,\Hlambda^{-1/2}U\big)
=\tr\big((I-UU^{\top})\hwsecond\big)
=\sum_{j=k+1}^d\mu_j.
\label{eq:app_optimal_value}
\end{equation}
This pair therefore attains the global minimum. Its stored representation and influence estimate are precisely
\begin{align}
V^{\top}\traingrad{i}&=U^{\top}\hwtrain_i,\\
\querygrad^{\top}MV^{\top}\traingrad{i}&=\hwquery^{\top}UU^{\top}\hwtrain_i
=\langle U^{\top}\hwquery,U^{\top}\hwtrain_i\rangle.
\end{align}
This proves \Cref{prop:optimality}.
\end{proof}

The result uses the uncentered second moment and remains valid when eigenvalues are repeated.

\subsection{Proof of Eigenbasis Correction}
\label[appendix]{app:eigenbasis_correction}

We prove \Cref{prop:rayleigh_ritz} within one module. Let $x_i=\Hekfac^{-1/2}\traingrad{i}$ have finite second moment, let $Q_m\in\mathbb{R}^{d\times m}$ have orthonormal columns, and set $a_i=Q_m^{\top}x_i$. Let $P\in\mathbb{R}^{m\times k}$ contain orthonormal leading eigenvectors of $\expectation_i[a_i a_i^{\top}]$, where $1\le k\le m$.

\begin{proof}
Every $k$-dimensional subspace of $\range(Q_m)$ has an orthonormal basis $U=Q_mW$ for some $W\in\mathbb{R}^{m\times k}$ satisfying $W^{\top}W=I_k$. Decomposing $x_i-UU^{\top}x_i$ into components orthogonal to and within $\range(Q_m)$ yields
\begin{equation}
\begin{aligned}
\expectation_i\|x_i-UU^{\top}x_i\|_2^2
&=\expectation_i\|(I-Q_mQ_m^{\top})x_i\|_2^2\\
&\quad+\expectation_i\|a_i-WW^{\top}a_i\|_2^2.
\end{aligned}
\label{eq:app_restricted_reconstruction}
\end{equation}
The first term is independent of $W$. By the PCA argument in \Cref{app:optimality}, the second is minimized by $W=P$. Thus $Q_mP$ minimizes reconstruction error over all $k$-dimensional subspaces of $\range(Q_m)$. The leading $k$ retained EK-FAC axes are among these feasible choices, proving part~(i).

For part~(ii), suppose $\range(Q_m)$ contains a globally optimal $k$-dimensional principal subspace of $x_i$ within the module. This subspace is feasible for the restricted problem, so the restricted minimum is at most the full-space minimum. It is also at least that minimum because the feasible set is restricted. The two errors therefore agree.
\end{proof}

\subsection{Scaled One-Bit Quantization}
\label[appendix]{app:quantization_scope}

\paragraph{Optimal scale for a fixed sign vector.}
For projected coordinates $a\in\mathbb{R}^k$ with $k\ge1$, let $b=\operatorname{sign}(a)\in\{\pm1\}^k$, allowing either sign for zero coordinates. For a scale $s\in\mathbb{R}$, the identities $b^{\top}a=\|a\|_1$ and $b^{\top}b=k$ give
\begin{equation}
\|a-sb\|_2^2
=\|a\|_2^2-2s\|a\|_1+ks^2
=k\left(s-\frac{\|a\|_1}{k}\right)^2+\|a\|_2^2-\frac{\|a\|_1^2}{k}.
\label{eq:app_onebit_scale}
\end{equation}
Thus, $s=\|a\|_1/k$ minimizes the squared reconstruction error for the sign vector $b$, including $s=0$ when $a=0$.
\FloatBarrier

\section{Algorithms and Projection Implementation}
\label[appendix]{app:algorithm}
\label[appendix]{app:exp_projection}

\subsection{Projection paths and unified procedure}

Both stages operate independently on the linear-layer modules defined in \Cref{sec:two_stage}. Across all evaluated models, \method{} applies PCA to the first-stage coordinates, while \methodrandom{} replaces PCA with an SRHT. The model-specific module lists, dimensions, and storage-budget allocations are recorded with the experimental settings.

For module $u$, the first-stage map and the subsequent projection are
\begin{equation}
L_u=(\Lambda_{u,m}+\lambda_u I)^{-1/2}Q_{u,m}^{\top},
\qquad
T_u(v)=
\begin{cases}
P_u^{\top}v, & \text{for \method{}},\\
\mathsf{S}_u v, & \text{for \methodrandom{}}.
\end{cases}
\label{eq:app_projection_maps}
\end{equation}
Here, $P_u\in\mathbb{R}^{m_u\times k_u}$ contains the PCA directions, and $\mathsf{S}_u\in\mathbb{R}^{k_u\times m_u}$ is an implicitly applied SRHT. The Kronecker factors and retained-coordinate indices implement $L_u$ without materializing $Q_{u,m}$ or $L_u$. We write $g_{i,u}$ and $q_u$ for the training and query gradient blocks of module $u$.

\begin{algorithm}[H]
\small
\caption{\method{} and \methodrandom{} fitting, storage, and querying}
\label{alg:eogp}
\begin{algorithmic}[1]
\Require Factorized EK-FAC bases $Q_u$, corrected eigenvalues $\Lambda_u$, damping values $\lambda_u>0$, dimensions $1\le k_u\le m_u\le d_u$, method variant, and a PCA fitting set $\mathcal C$ for \method{}
\Statex \textbf{Fitting shared projections}
\For{each module $u$}
    \State Select $Q_{u,m},\Lambda_{u,m}$ using the $m_u$ largest corrected eigenvalues
    \State For \methodrandom{}, initialize the implicit SRHT $\mathsf{S}_u$
\EndFor
\If{using \method{}}
    \For{each $i\in\mathcal C$}
        \State Compute the training gradient $g_i$ once
        \For{each module $u$}
            \State Compute and cache $z_{i,u}\gets L_u g_{i,u}$
        \EndFor
    \EndFor
    \State For each $u$, fit $P_u$ by uncentered PCA of $\{z_{i,u}:i\in\mathcal C\}$
\EndIf
\Statex \textbf{Storing one training example $i$}
\State Compute $g_i$ once
\For{each module $u$}
    \State $x_{i,u}\gets T_u(L_u g_{i,u})$
    \State Store packed binary signs $b_{i,u}\gets\operatorname{sign}_{\pm}(x_{i,u})$ and $s_{i,u}\gets\operatorname{FP16}(\|x_{i,u}\|_1/k_u)$
\EndFor
\Statex \textbf{Scoring a query $f$}
\State Compute $q=\nabla_\theta f(\theta^*)$ once
\State For each $u$, compute the unquantized coordinates $y_u\gets T_u(L_u q_u)$
\State \Return $\widehat{\mathcal I}(i)=\sum_u s_{i,u}\langle y_u,b_{i,u}\rangle$ for each stored example $i$
\end{algorithmic}
\end{algorithm}

The batched implementation is described in \Cref{app:timing}.

\subsection{PCA fitting and random projection}
\label[appendix]{app:srht}

\paragraph{PCA fitting.} For \method{}, we form the first-stage coordinates $z_{i,u}=L_ug_{i,u}$ using training-data labels and apply PCA without subtracting the sample mean. Let $Z_u\in\mathbb{R}^{n\times m_u}$ contain these coordinates as rows, where $n=|\mathcal C|$. The columns of $P_u$ are the top-$k_u$ right singular vectors of $Z_u$, equivalently the eigenvectors associated with the largest eigenvalues of $Z_u^{\top}Z_u/n$. The nonzero principal directions can be recovered from the eigendecomposition of the $n\times n$ Gram matrix $Z_uZ_u^{\top}$, avoiding construction of an $m_u\times m_u$ second-moment matrix. This is the PCA correction analyzed in \Cref{prop:rayleigh_ritz}.

\paragraph{Random projection.} For \methodrandom{}, a subsampled randomized Hadamard transform (SRHT)~\citep{tropp2011improved} maps the first-stage coordinates directly to $k_u$ dimensions,
\begin{equation}
\mathsf{S}_u
=
\sqrt{\frac{p_u}{k_u}}\,R_u H_{p_u}D_uJ_u,
\qquad
p_u=2^{\lceil\log_2 m_u\rceil}.
\label{eq:app_srht_definition}
\end{equation}
Here, $J_u$ zero-pads the $m_u$ coordinates to width $p_u$, $D_u$ is a diagonal matrix of independent random signs, $H_{p_u}$ is the orthonormal Walsh--Hadamard matrix, and $R_u$ selects $k_u$ rows uniformly without replacement. The same realized transform is used for gradient storage and query projection. This variant does not fit a PCA correction.

\subsection{Computation and complexity}
\label[appendix]{app:computation}

\paragraph{Query scoring.} Using the unquantized query coordinates $y_u=T_u(L_uq_u)$ from \Cref{alg:eogp} and the stored signs $b_{i,u}$ and scales $s_{i,u}$, the influence score is
\begin{equation}
\widehat{\mathcal{I}}(i)
=
\sum_u s_{i,u}\langle y_u,b_{i,u}\rangle.
\label{eq:app_stored_influence}
\end{equation}
The same query coordinates are reused to score all stored examples.

\paragraph{Projection cost.} Applying the PCA correction in \method{} costs $O(m_uk_u)$ operations per module. The implicit SRHT in \methodrandom{} costs $O(p_u\log p_u)$ operations per module, where $p_u$ is the padded dimension in \Cref{eq:app_srht_definition}.

\paragraph{Storage and scanning.} Computing signs and scales from an already projected training representation costs $O(\sum_u k_u)$ operations. After projecting a query, scoring all $N_{\mathrm{store}}$ examples costs
\begin{equation}
O\!\left(N_{\mathrm{store}}\sum_u k_u\right)
\label{eq:app_scan_complexity}
\end{equation}
arithmetic operations. Storage in bytes and measured runtimes are reported in \Cref{app:storage_accounting,app:timing}.
\FloatBarrier

\section{Experimental Configuration}
\label[appendix]{app:experimental_details}
\label[appendix]{app:experimental_configuration}

This section describes the configurations used in \Cref{sec:experiments}.

\subsection{Models, data, and gradient computation}
\label[appendix]{app:models_data}

We evaluate GPT-2 on WikiText-2 and released OLMo~2 SFT checkpoints on T\"ulu~3 data. The GPT-2 attribution checkpoint is fine-tuned using the protocol in \Cref{app:exp_retraining}. \Cref{tab:app_model_configuration} lists the source checkpoints and datasets.

\begin{table}[t]
\centering
\small
\caption{Source model checkpoints and datasets.}
\begin{tabularx}{\linewidth}{@{}p{0.17\linewidth}X p{0.22\linewidth}@{}}
\toprule
Model & Source checkpoint & Dataset \\
\midrule
GPT-2
& \texttt{gpt2}
& WikiText-2 \\
OLMo~2 1B
& \texttt{allenai/OLMo-2-0425-1B-SFT}
& T\"ulu~3 mixture \\
OLMo~2 7B
& \texttt{allenai/OLMo-2-1124-7B-SFT}
& T\"ulu~3 mixture \\
OLMo~2 13B
& \texttt{allenai/OLMo-2-1124-13B-SFT}
& T\"ulu~3 mixture \\
OLMo~2 32B
& \texttt{allenai/OLMo-2-0325-32B-SFT}
& T\"ulu~3 mixture \\
\bottomrule
\end{tabularx}

\label{tab:app_model_configuration}
\end{table}

\paragraph{Data preparation.}
We tokenize \texttt{wikitext-2-raw-v1} from \texttt{Salesforce/wikitext} with the GPT-2 tokenizer and divide concatenated text into 512-token blocks. Training and validation blocks provide attribution candidates and queries, respectively. OLMo~2 1B and 32B use \path{allenai/tulu-3-sft-olmo-2-mixture-0225}, while 7B and 13B use \path{allenai/tulu-3-sft-olmo-2-mixture}. Each conversation is one example, with queries selected separately from attribution candidates. Evaluation protocols are given in \Cref{app:evaluation_protocols}.

\paragraph{Loss and query measurement.}
Both training and query gradients use token-summed next-token cross entropy. GPT-2 includes all prediction positions, while OLMo~2 includes only assistant-response positions and their turn-ending tokens.

\paragraph{Attributed parameters.}
We attribute to the attention and MLP linear layers, including biases in GPT-2. 
Embeddings, the language-model output head, and normalization parameters are excluded. Attribution gradients are computed in evaluation mode.

\subsection{Curvature fitting and projection dimensions}
\label[appendix]{app:projection_configuration}

\paragraph{Fitting configuration.} We estimate EK-FAC statistics using Kronfluence. 
Fisher statistics use labels sampled from the model's next-token distribution with input contexts fixed to the data. Factor statistics are averaged over non-padding tokens, and corrected eigenvalues over examples using sequence-level gradients.

For each module $u$ with $d_u$ attributed parameters, we use relative damping
\begin{equation}
\lambda_u
=
0.1\,\frac{\operatorname{tr}(\Lambda_u)}{d_u},
\label{eq:app_module_damping}
\end{equation}
where $\Lambda_u$ is the diagonal matrix containing the module's full set of corrected eigenvalues. The resulting half-whitening transformation is applied to both training and query gradients.

\paragraph{PCA fitting data.} For OLMo~2, the PCA fitting set $\mathcal{C}$ contains 10,000 conversations sampled from the same SFT mixture with a fixed random seed, after excluding all candidate, query, and curvature-fitting examples. Each conversation is truncated to at most 2,048 tokens.

\paragraph{Projection dimensions.} \method{} applies PCA to the first-stage coordinates on all models. For OLMo~2, $m_u=262{,}144$, and the final dimension is uniform across modules, ranging in powers of two from 256 to 8,192. \methodrandom{} uses the same first-stage width and between 256 and 32,768 SRHT output coordinates.

\paragraph{Storage budgets.}
Budgets count per-example coordinates and scales, with projection dimensions chosen to match storage across methods.
In the main comparisons, \method{} and \method{}-R use one-bit coordinates with FP16 scales, while compressed baselines use FP16.
\Cref{fig:ablation} compares both precisions; in its right panel, which fixes coordinate counts on OLMo~2 7B, the one-bit curves also form a matched-storage comparison, since equal coordinate counts imply equal one-bit storage, and \method{} outperforms all one-bit baselines.
GPT-2 follows the nominal budget labels in \Cref{fig:lds_counterfactual}, and OLMo~2 reports payloads in decimal units.
Shared artifacts are accounted for separately in \Cref{app:storage_accounting}.

\subsection{Baseline configurations}
\label[appendix]{app:baseline_configuration}

All methods share attributed modules, evaluation examples, queries, and loss definitions. For OLMo~2, \method{}, LoGra, GraSS, and LoRIF fit their curvature estimates on the same 10,000 conversations. Storage formats and budget matching follow \Cref{app:projection_configuration}.

\paragraph{Compression baselines.} Our LoGra implementation is based on LogIX\footnote{\url{https://github.com/logix-project/logix}}, with random and PCA initialization on both GPT-2 and OLMo~2. LoGra uses the same relative damping coefficient ($0.1$) as in \Cref{app:projection_configuration}. PCA initialization uses factor eigenvectors fitted with Kronfluence\footnote{\url{https://github.com/pomonam/kronfluence}}, and the factor rank $a$ determines $a^2$ stored coordinates per module. Our GraSS implementation follows the official code, using factor-wise random masking and CountSketch with output dimensions matched to LoGra. Our LoRIF reimplementation uses randomly projected gradient matrices and a rank-256 inverse-curvature approximation.

\paragraph{Uncompressed references.} EK-FAC and K-FAC use Kronfluence with model-sampled labels and the damping in \Cref{app:projection_configuration}. The TracIn reference is the unweighted gradient inner product at the final checkpoint.

\paragraph{Gaussian random projection.} This control uses independent standard Gaussian factors to produce $a^2$ coordinates per module. Training gradients are projected directly, and query gradients are preconditioned by the full-space damped EK-FAC inverse before projection. Inner products are scaled by $a^{-2}$.

\paragraph{Randomization.} LoGra and LoRIF use one projection realization per configuration. GraSS uses one on GPT-2 and averages over three seeds on OLMo~2. The Gaussian control averages over three independent draws, and \methodrandom{} over up to three disjoint blocks of a shared SRHT output. Retraining and uncertainty estimation follow \Cref{app:exp_retraining}.
\FloatBarrier

\section{Evaluation Protocols}
\label[appendix]{app:evaluation_protocols}
\label[appendix]{app:metrics}

\subsection{GPT-2 retraining evaluation}
\label[appendix]{app:exp_retraining}

We evaluate attribution on WikiText-2 using the linear datamodeling score (LDS) and counterfactual retraining.

\paragraph{Linear datamodeling score.}
We reuse the publicly released subset masks and retraining losses from Kronfluence for $R=100$ subsets, each retaining half of the training blocks. All methods and storage budgets use the same subsets and losses on the validation blocks. The query measurement is the token-summed next-token cross-entropy loss of a validation block.

Let $S_r$ denote the retained training subset and $Y_{rq}$ its released retraining loss for query $q$. Under our removal-based influence convention, the predicted subset score is
\begin{equation}
A_{rq}
=
-\sum_{i\in S_r}\widehat{\mathcal I}_q(i).
\label{eq:app_lds_measurements}
\end{equation}
For each query, we compute the Spearman correlation across subsets,
\begin{equation}
\operatorname{LDS}_q
=
\rho_{\mathrm S}
\left(
(A_{rq})_{r=1}^{R},
(Y_{rq})_{r=1}^{R}
\right),
\label{eq:app_lds}
\end{equation}
and report its mean over queries, excluding undefined correlations. We use average ranks for ties. We obtain 95\% percentile confidence intervals from 1,000 bootstrap resamples of the subsets, recomputing the query-averaged LDS for each resample.

\paragraph{Counterfactual retraining.}
For each method and storage budget, we rank the training blocks by the sum of their signed influence scores over the first 50 validation blocks. We remove the examples with the largest scores, reusing the same ranking for removal counts from 50 to 300 in increments of 50.

Each retraining run starts from the pretrained GPT-2 weights and uses AdamW with a learning rate of $3\times10^{-5}$, weight decay of $0.01$, three epochs, and a training batch size of eight. We use a set $\mathcal S$ of five seeds, matched across methods. Perplexity is evaluated on the same 50 validation blocks used to construct the ranking by exponentiating their mean token-level cross-entropy loss.

Let $\operatorname{PPL}_s(b)$ denote the perplexity for seed $s$ after removing $b$ examples. We report
\begin{equation}
\Delta\operatorname{PPL}(b)
=
\frac{1}{5}
\sum_{s\in\mathcal S}
\left[
\operatorname{PPL}_s(b)
-
\operatorname{PPL}_s(0)
\right],
\label{eq:app_counterfactual}
\end{equation}
where the zero-removal runs are shared across methods. Error bars show the population standard deviation of the post-removal perplexities across seeds. For the random-removal control, each seed defines a random permutation of the training set, and the first $b$ examples are removed.

\subsection{Fidelity to EK-FAC influence}
\label[appendix]{app:exp_llms}

We compare compressed influence scores with uncompressed EK-FAC scores on the same candidate training examples. OLMo~2 uses 1,000 candidates and 100 queries, while GPT-2 uses the training and validation blocks as candidates and queries, respectively. Module contributions are summed before evaluation, and both score vectors use signed scores under the same removal-based convention, without taking absolute values or clipping.

\paragraph{Spearman correlation.}
For query $q$, let $\mathbf e_q$ and $\mathbf a_q$ denote the reference and compressed score vectors over the candidate examples. We compute
\begin{equation}
\rho_q
=
\operatorname{Corr}
\left(
\operatorname{rank}(\mathbf a_q),
\operatorname{rank}(\mathbf e_q)
\right),
\qquad
\overline{\rho}_{\mathrm S}
=
\frac{1}{|\mathcal Q_{\mathrm{valid}}|}
\sum_{q\in\mathcal Q_{\mathrm{valid}}}\rho_q.
\label{eq:app_spearman}
\end{equation}
Ties receive average ranks. A correlation is undefined when either score vector is constant, and $\mathcal Q_{\mathrm{valid}}$ contains the queries with defined correlations.

\paragraph{NDCG@20.}
We use binary relevance based on membership in the reference top 20. Let $T_q$ contain the 20 examples with the largest signed EK-FAC scores, and let $\pi_q(t)$ denote the example at position $t$ when compressed scores are sorted in descending order. Then
\begin{equation}
\operatorname{NDCG@20}(q)
=
\frac{
\displaystyle\sum_{t=1}^{20}
\frac{\mathbf{1}\{\pi_q(t)\in T_q\}}{\log_2(t+1)}
}{
\displaystyle\sum_{t=1}^{20}
\frac{1}{\log_2(t+1)}
}.
\label{eq:app_ndcg}
\end{equation}
The denominator is the ideal discounted gain when all 20 relevant examples occupy the first 20 positions. We report the arithmetic mean of NDCG@20 over queries.

\subsection{Relationship between fidelity and retraining}
\label[appendix]{app:fidelity_lds_relationship}

\Cref{fig:lds_vs_EKFAC} compares EK-FAC fidelity with LDS across the 25 compressed configurations and three reference methods in \Cref{fig:lds_counterfactual}. For each configuration, all metrics use the same attribution score matrix. We average each metric over queries before computing correlations across configurations.

Let $\mathcal{J}$ denote these 28 configurations, and let $L_j$, $N_j$, and $S_j$ denote their query-averaged LDS, NDCG@20, and Spearman fidelity. The resulting correlations are
\begin{equation}
\begin{aligned}
\rho_{\mathrm{S}}\!\left((L_j)_{j\in\mathcal{J}},(N_j)_{j\in\mathcal{J}}\right)
&\approx 0.95,\\
\rho_{\mathrm{S}}\!\left((L_j)_{j\in\mathcal{J}},(S_j)_{j\in\mathcal{J}}\right)
&\approx 0.84.
\end{aligned}
\label{eq:app_fidelity_lds}
\end{equation}
\FloatBarrier

\section{Storage and Runtime Accounting}
\label[appendix]{app:cost}

\subsection{Per-example payload and shared artifacts}
\label[appendix]{app:storage_accounting}

We account separately for per-example representations and artifacts shared across the gradient store. Let $\mathcal{U}$ denote the attributed modules, $m_u$ and $k_u$ the first-stage and final output dimensions of module $u$, respectively, and $N_{\mathrm{store}}$ the number of stored examples. The attributed parameter set is specified in \Cref{app:models_data}.

\paragraph{Per-example payload.} \method{} and \methodrandom{} store one sign bit per coordinate and one FP16 scale per module. With byte packing performed separately for each module, the payload is
\begin{equation}
B_{\mathrm{1bit}}
=
\sum_{u\in\mathcal{U}}
\left(
\left\lceil\frac{k_u}{8}\right\rceil+2
\right)
\quad\text{bytes}.
\label{eq:app_payload}
\end{equation}
For a uniform dimension $k$, this becomes $|\mathcal{U}|(\lceil k/8\rceil+2)$. Storing the coordinates directly in FP16 requires $B_{\mathrm{FP16}}=2\sum_u k_u$ bytes, while an uncompressed FP16 gradient requires $B_{\mathrm{raw}}=2d_{\mathrm{attr}}$ bytes, where $d_{\mathrm{attr}}$ is the number of attributed parameters.

\paragraph{Shared artifacts.} The Kronecker eigenbases and half-whitening weights are shared across examples and queries. In \method{}, each module's PCA correction matrix $P_u\in\mathbb{R}^{m_u\times k_u}$ is also shared and stored in FP16, requiring
\begin{equation}
B_{P,u}
=
2m_uk_u
\quad\text{bytes}.
\label{eq:app_correction_storage}
\end{equation}
For the 32B model with $m_u=262{,}144$ and $k_u=2{,}048$, each FP16 correction matrix requires approximately 1.07~GB, totaling 481~GB across all 448 modules. This cost is independent of the number of stored examples, so its relative contribution decreases as the store grows. For stores containing 100 million or one billion examples, the correction matrices would occupy approximately 4.2\% or 0.42\% of the aggregate per-example payload, respectively. These estimates illustrate how the shared correction matrices can be amortized when constructing attribution stores for large pretraining or mid-training corpora.

When shared storage is constrained, reducing $k_u$ decreases the correction-matrix storage linearly. For example, keeping $m_u$ unchanged and reducing $k_u$ to 256 lowers the total correction-matrix storage to approximately 60.1~GB for the 32B model. As shown in \Cref{fig:ekfac}, \method{} maintains an advantage over the evaluated compression baselines at small per-example storage budgets.

\paragraph{Total storage.} Including shared artifacts and file overhead, the total storage is
\begin{equation}
B_{\mathrm{total}}
=
N_{\mathrm{store}}B_{\mathrm{1bit}}
+
B_{\mathrm{shared}}
+
B_{\mathrm{overhead}},
\label{eq:app_total_storage}
\end{equation}
where $B_{\mathrm{shared}}$ counts shared curvature and projection artifacts, including $\sum_{u\in\mathcal{U}}B_{P,u}$ for the PCA correction matrices in \method{}, and $B_{\mathrm{overhead}}$ counts file metadata. Per-example budgets count representation payloads, with shared artifacts accounted for separately.

\subsection{Storage estimates by model}
\label[appendix]{app:reported_costs}

\Cref{tab:cost} compares per-example storage for uncompressed FP16 gradients and one-bit representations at $k_u=2{,}048$. Both \method{} and \methodrandom{} use $2048/8+2=258$ bytes per module per example.

\begin{table}[t]
\centering
\small
\caption{Per-example representation payloads at $k_u=2{,}048$ per module. Uncompressed gradients cover the attributed parameters, and one-bit representations include packed signs and FP16 scales. Shared artifacts are accounted for in \Cref{app:storage_accounting}.}
\begin{tabular}{@{}lrr@{}}
\toprule
Model & \shortstack{Uncompressed FP16\\(GB/example)} & \shortstack{One-bit representation\\(KB/example)} \\
\midrule
OLMo~2 1B  &  2.147 &  28.896 \\
OLMo~2 7B  & 12.952 &  57.792 \\
OLMo~2 13B & 25.376 &  72.240 \\
OLMo~2 32B & 62.411 & 115.584 \\
\bottomrule
\end{tabular}

\label{tab:cost}
\end{table}

\subsection{Implementation and runtime}
\label[appendix]{app:timing}

\paragraph{Setup.} \Cref{tab:app_timing_record} compares our implementation with LoGra's official LogIX implementation on a single NVIDIA B200. \method{} uses $m_u=262{,}144$, $k_u=2{,}048$, and FP16 correction matrices, while LoGra stores $2{,}025$ FP16 coordinates per module. Curvature statistics are fitted beforehand, and timings exclude curvature fitting and model loading.

\begin{table}[t]
\centering
\small
\caption{PCA fitting and store-construction times on one NVIDIA B200. Fitting starts from cached coordinates and includes saving the correction matrices. Store-construction totals combine separately timed stages, including disk I/O.}
\begin{tabular}{@{}lrrr@{}}
\toprule
Model
& \shortstack{\method{} PCA fitting\\(min)}
& \multicolumn{2}{c}{Store construction (s/example)} \\
\cmidrule(l){3-4}
& & \method{} & LoGra \\
\midrule
OLMo~2 1B  & 4.5 & 0.0563 & 0.0577 \\
OLMo~2 7B  & 8.9             & 0.0999 & 0.1315 \\
OLMo~2 13B & 11.1            & 0.1474 & 0.1661 \\
OLMo~2 32B & 18.4            & 0.2470 & 0.2780 \\
\bottomrule
\end{tabular}
\label{tab:app_timing_record}
\end{table}

\paragraph{Implementation.} First-stage coordinates are cached in FP16 using asynchronous disk writes. We then process modules sequentially, loading each correction matrix onto the GPU and applying it to a batch of coordinates before one-bit encoding and recording. Each loaded matrix is reused across the batch.

\paragraph{Timing protocol.} Store construction includes gradient computation, projection, and disk I/O, including correction-matrix loading and transfer for \method{}. Coordinate generation is timed on 40 documents per model averaging approximately 500 tokens. Subsequent PCA projection and storage use batches of approximately 11,000 documents, over which matrix-loading costs are amortized.

\begin{figure}[t]
\centering
\includegraphics[width=\linewidth]{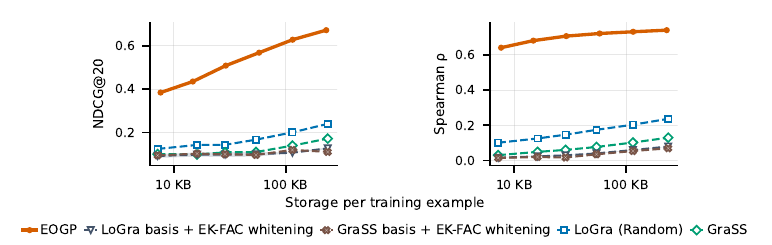}
\caption{Comparison of \method{} with LoGra and GraSS on OLMo~2 7B SFT, using either their original curvature treatment or EK-FAC half-whitening before projection. NDCG@20 (left) and Spearman correlation (right) measure agreement with EK-FAC influence at matched per-example storage budgets.}
\label{fig:same_curvature}
\end{figure}

\paragraph{One-time PCA fitting.} Fitting starts from cached first-stage coordinates of approximately 10,000 examples and follows \Cref{app:srht}. Timings include reading the coordinates, computing the correction matrices, and saving them in FP16. This one-time cost is 18.4 minutes for the 32B model and is shared across examples and queries.

\paragraph{Batched query processing.} For the 32B model with $m_u=262{,}144$ and $k_u=2{,}048$, processing 100 queries against 1,000 stored examples takes 60.36 seconds on a single NVIDIA B200, starting from cached first-stage coordinates. This includes loading and transferring the FP16 correction matrices, applying the PCA projection, and scoring the stored examples. Each correction matrix is loaded once per batch, giving an amortized processing cost of 0.604 seconds per query.

\FloatBarrier

\section{Comparison under Shared EK-FAC Curvature}
\label[appendix]{app:shared_curvature}

\begin{table}[t]
    \centering
    \caption{Pearson correlation between FP16 and one-bit influence
    scores on GPT-2. We report the median of the query-wise
    correlations over 481 queries. Storage budgets refer to
    one-bit storage per training example.
    The highest correlation at each budget is shown in bold.}
    \label{tab:onebit_score_correlation}
    \small
    \setlength{\tabcolsep}{10pt}
    \renewcommand{\arraystretch}{1.1}
    \begin{tabular}{lrrr}
        \toprule
        \textbf{Method}
        & \textbf{12 KB}
        & \textbf{24 KB}
        & \textbf{96 KB} \\
        \midrule
        EOGP
        & $\mathbf{0.902}$ & $\mathbf{0.890}$ & $\mathbf{0.919}$ \\
        LoGra (Random)
        & $0.604$ & $0.447$ & $0.086$ \\
        LoGra (PCA)
        & $0.750$ & $0.570$ & $-0.355$ \\
        GraSS
        & $0.330$ & $0.079$ & $-0.205$ \\
        LoRIF
        & $0.820$ & $0.823$ & $0.822$ \\
        \bottomrule
    \end{tabular}
\end{table}
\begin{figure}[t]
\centering
\includegraphics[width=.9\linewidth]{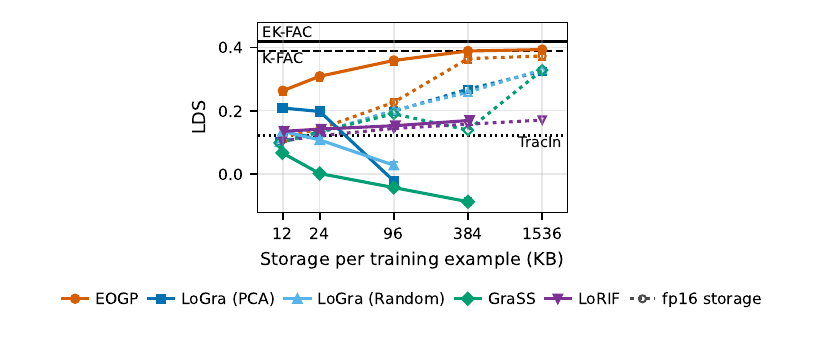}
\caption{LDS of compression methods under one-bit and FP16 storage on GPT-2 at matched per-example storage budgets. Horizontal lines denote uncompressed references. Curves end where the required projection dimension exceeds the method's supported dimension.}
\label{fig:onebit_lds}
\end{figure}

To examine whether \method{} benefits primarily from using the same curvature approximation as the evaluation reference, we construct variants of LoGra and GraSS that apply our EK-FAC half-whitening before their respective projections. On OLMo~2 7B SFT, \method{} substantially outperforms both variants at matched per-example storage budgets, showing that matching the reference curvature alone does not explain its advantage.

The original LoGra and GraSS configurations also generally outperform their half-whitened variants, supporting the effectiveness of their post-projection curvature approximations in this setting. One possible explanation is that estimating curvature after projection accounts for the variances and correlations of the retained coordinates, whereas applying half-whitening before projection produces a different transformation. These results suggest that the effectiveness of a curvature approximation depends on how it is combined with gradient compression.

\FloatBarrier
\section{Effect of One-Bit Quantization on Attribution Quality}
\label[appendix]{app:onebit_analysis}

We further examine how one-bit storage affects different gradient representations using the GPT-2 setting in \Cref{sec:exp_retraining}. We compare attribution quality at matched per-example storage budgets and separately measure how quantization changes influence scores while holding the projection dimension fixed.

\paragraph{Attribution quality under one-bit storage.}
\Cref{fig:onebit_lds} compares one-bit and FP16 storage for the evaluated compression methods. For baselines that estimate curvature after projection, larger output dimensions increase the memory required for curvature estimation, so we evaluate their one-bit variants only up to the largest feasible storage budget. Under one-bit storage, \method{} achieves the highest LDS among the evaluated compression methods at each displayed budget, with performance improving as more storage is allocated. The one-bit variants of LoGra and GraSS exhibit a different trend, with LDS declining as the budget increases. LoRIF improves more gradually but remains below \method{}. At larger storage budgets, \method{} also outperforms the other compression baselines when all methods use FP16 storage, indicating that its advantage is not solely attributable to one-bit quantization. These results show that the effectiveness of one-bit storage depends on the representation being quantized, and that \method{}'s advantage persists in a retraining-based evaluation.

\paragraph{Preservation of influence scores.}
To examine the effect of quantization directly, we compare influence scores obtained with FP16 and one-bit storage at the same projection dimension within each method. For each query, we compute the Pearson correlation between the two sets of scores across candidate training examples, then report the median correlation over queries in \Cref{tab:onebit_score_correlation}.

For \method{}, these correlations range from $0.890$ to $0.919$ across the evaluated configurations, indicating that the scores remain strongly correlated after quantization. LoGra and GraSS show substantially weaker agreement in several configurations. These correlations measure the preservation of score patterns rather than absolute score magnitudes, since Pearson correlation is unchanged by a positive rescaling or an additive shift. Together with the LDS results, they provide complementary evidence that \method{} preserves useful attribution information under one-bit storage.

\paragraph{A possible explanation.}
One possible explanation is that half-whitening makes the representation more compatible with scaled sign quantization. This quantization preserves coordinate signs but replaces their magnitudes with a common scale given by the mean absolute coordinate value. It therefore approximates vectors more accurately when their coordinate magnitudes are relatively uniform. By rescaling gradients using curvature information, half-whitening may reduce magnitude imbalances in the resulting coordinates and thereby lessen the distortion introduced by quantization. The relevant property is the relative variation in coordinate magnitudes, since uniformly rescaling a vector also rescales its quantization scale and leaves the relative reconstruction error unchanged.

\end{document}